\documentclass[letterpaper, 10 pt, conference]{ieeeconf}

\IEEEoverridecommandlockouts                              

\usepackage{multicol}

\usepackage{threeparttable}
\usepackage{dsfont}
\usepackage{amsmath,amsfonts,amssymb}
\usepackage{graphicx,subcaption,wrapfig}
\usepackage{bm}
\newcommand*{\xhat}[1]{#1\kern-0.35em\hat{\phantom{#1}}}
\usepackage[ruled,longend]{algorithm2e}
\usepackage{mathrsfs,multicol}
\usepackage{amstext}
\usepackage{float}
\usepackage{fullpage}
\usepackage{nomencl}

\usepackage{algorithmic}

\newcommand{\T}{\ensuremath{^\mathsf{T}}}

\newtheorem{theorem}{Theorem}[section]
\newtheorem{proposition}{Proposition}[section]

\newtheorem{assumption}{Assumption}
\newtheorem{lemma}{Lemma}[section]

\usepackage[maxbibnames=10,giveninits=true,sorting=none]{biblatex} 
\AtEveryBibitem{%
  \clearfield{url}
   \clearfield{organization}
}
\DeclareFieldFormat*{year}{{}}
\renewbibmacro{in:}{} 

\title{Learning to Control under Uncertainty with Data-Based Iterative Linear Quadratic Regulator}
\author{Ran Wang, Raman Goyal, and Suman Chakravorty   
\thanks{The authors are with the Department of Aerospace Engineering, Texas A\&M University, College Station, TX 77843 USA. \{\tt rwang0417, ramaniitrgoyal92, schakrav\}@tamu.edu
}}

\begin{document}

\maketitle  

\begin{abstract}
This paper studies the learning-to-control problem under process and sensing uncertainties for dynamical systems. 
In our previous work, we developed a data-based generalization of the iterative linear quadratic regulator (iLQR) to design closed-loop feedback control for high-dimensional dynamical systems with partial state observation. This method required perfect simulation rollouts which are not realistic in real applications. In this work, we briefly introduce this method and explore its efficacy under process and sensing uncertainties. We prove that in the fully observed case where the system dynamics are corrupted with noise but the measurements are perfect, it still converges to the global minimum. However, in the partially observed case where both process and measurement noise exist in the system, this method converges to a biased ``optimum". Thus multiple rollouts need to be averaged to retrieve the true optimum. The analysis is verified in two nonlinear robotic examples simulated in the above cases.
\end{abstract} 

\section{Introduction}
The optimal control of a nonlinear dynamical system is computationally intractable for complex high-dimensional systems owing to the ``curse of dimensionality"  \cite{dp_bertsekas} and the problem becomes even more formidable when the system has no known analytical model and is under partial state observation. In fact, most real-world problems are partially observed, which has been recognized as one of the major gaps that keep controllers designed in simulation from being applied successfully to real-world applications \cite{DRLreview}.

There has been significant work in the field of learning to control unknown dynamical systems using Reinforcement Learning (RL), with great progress in creating accurate models for complex robots \cite{mnih2015human}. Despite excellent performance on several tasks \cite{levine2016end}, most of the work is in simulation, and applying RL to real robots remains challenging \cite{DRLreview}. The performance of policies trained in simulation directly applied to real robots can be poor due to the ``sim-to-real" gap for numerous reasons. First, it is impossible to capture all the physics with the process and sensing uncertainties in a simulation model and the simulated sensor data can be very different from its real-world counterpart. 
Most importantly, as full-state measurements may not be available in real-world robots, the policy has to be trained with partial observation, which poses challenges to the RL algorithms. 


\noindent \underline{\textit{\textbf{Related Work:}}}
RL researchers have made progress in applying RL to the real world in both sim-to-real and learning on real robot directions. Previous work such as \cite{Li2021cassie} adopted domain randomization to address uncertainty and latency,
which improves the robustness, but harms the optimality. \cite{Rudin2021} parallelized the simulations on multiple agents to decrease total training time and improve the robustness. Also, adapter networks are used to make simulated images more similar to their real-world counterparts \cite{rao2020} or the other way around \cite{james2019}. With these improvements, the trained policy can be directly applied to real robots with comparable performance \cite{Surmann2020DeepRL}. 

In the direction of on-robot training, multiple learning processes can be run with different hyper-parameters on the same robot \cite{khadka2019}, 
which greatly shorten the tuning process. To tackle the resetting issue, a controller can be designed to reset the robot \cite{ingredientofrw} at the end of each rollout.
This method requires expert knowledge to design the resetting controller, which could be challenging for complex robots. In fact, on-robot training requires resetting heavily due to the finite time rollouts needed by RL algorithms. 
To take advantage of both simulation and on-robot training, one can prototype a policy in simulation and improve it with a relatively small amount of online training \cite{QTOPT}. 

We proposed a data-based learning-to-control approach in our previous work \cite{POD2C} for partially observed applications. In this paper, we focus on the problem of learning to control under uncertainty using the aforementioned partially observed data-based iLQR (POD-iLQR) algorithm. POD-iLQR is a data-based generalization of iLQR for partially observed problems. It converts partially observed problems to ``fully observed" problems using a suitably defined information state and achieves high training efficiency by decoupling the open-loop and feedback design \cite{cdc2021}. The information state-based Linear Time-Varying (LTV) Autoregressive–Moving-Average (ARMA) system identification method is used to estimate the ``fully observed" linearized information state model, which allows us to solve the optimal control problem while the system model is unknown. We have shown that our generalized iLQR method tackles the challenge of partial observation in a highly efficient fashion. In this paper, its performance under process and sensing uncertainties will be studied. The \textbf{\textit{primary contributions}} of this paper are as follows. We study the learning under noise problem with the generalized iLQR method mentioned above. We analyze its efficacy when the training is carried out under process and measurement noise. We prove its convergence to the global minimum in the full state observation case, i.e., when the state is measured perfectly but the system dynamics are corrupted with noise. We show that it constructs biased LTV systems and can not converge to the true optimum in the noisy partially observed case and that multiple rollouts need to be averaged to recover global optimality. The reason we use the POD-iLQR is that we have already shown that the basic fully observed approach is superior to state-of-the-art RL techniques \cite{cdc2021}, while at the same time, there are no generalizations of typical RL approaches to partially observed problems \cite{POD2C}. Albeit we test in simulations on relatively simple systems, nonetheless, this work allows us to clearly show the challenges of learning under uncertainty in the partially observed case, which we can only expect to be aggravated when trying to learn on a real system. Also, the analysis of the simulation results helps us further understand the convergence of POD-iLQR under uncertainty, which could help the development of the modified POD-iLQR algorithm and its deployment on real systems eventually.

The rest of the paper is organized as follows: Section II provides the optimal control problem formulation. Section III briefly introduces our generalized iLQR approach. Section IV is the main focus of this paper and generalizes the POD-iLQR to the fully observed and partially observed cases under noise. Section V analyzes the performance of POD-iLQR directly applied in fully observed and partially observed problems under uncertainty. Empirical results are shown to support the results in Section VI.

\section{Partially Observed Data-based iLQR (POD-iLQR)}\label{s:Inropod2c}
Let us start by writing the stochastic nonlinear dynamics in discrete time state space form as follows:
$
x_{t+1} = f(x_t,u_t)+\omega_t,
$
where $x_t$ is the state, $u_t$ is the control input of the system and $\omega_t$ is the process noise. Let us assume the observation model to be of the form:
$
 z_{t} = h(x_t)+v_t,
$
where $z_t$ is the measurement and $v_t$ is the sensor noise.
Let us now define a finite horizon objective function as:\\
$
    J(z_0) = E\left[\sum_{t=0}^{T-1}c(z_t,u_t) + c_T(z_T)\right],
$
where $c(z_t,u_t)$ denotes a running incremental cost and $c_T(z_T)$ denotes a terminal cost function. The POD-iLQR was proposed to find the control policy to minimize the cost function above with deterministic system dynamics and measurements, i.e., $\omega_t$ and $z_t$ are zero. The goal of this work is to apply POD-iLQR under uncertainty and analyze its efficacy. 


In this section, we briefly introduce the main components of the POD-iLQR generalization to data-based partially observed problems and present the algorithm that we will analyze ``under noise" in the next section. The detailed algorithm and the global optimal solution analysis can be found in our previous work \cite{POD2C}.

\subsection{The Global Optimal Solution for the Partially Observed Problem}\label{s:infostate}
Let $Z_t^q = [z_{t-q}^T,z_{t-q+1}^T,\cdots,z_t^T]^T$ and $U_t^q = [u_{t-q}^T,u_{t-q+1}^T,\cdots,u_{t-1}^T]^T$, where $q$ is the number of outputs/inputs included in the information state. We make the following assumption:
\begin{assumption}
\textit{Observability:}
\label{a:obsv}
We assume that there exists a finite $\bar{q}$, such that for all $q\geq \bar{q}$, equation ${x}_{t-q} = \bar{f}(Z_t^q,U_t^q)$ has a unique solution for $x_{t-q}$, regardless of $(Z_t^q,U_t^q)$, where $ \bar{f}$ is the system dynamics w.r.t. $Z_t^q$ and $U_t^q$.
\end{assumption}
Let us now define the ``\textit{Information State}" $\mathcal{Z}_t^q$ at time $t$ as: 
$
\mathcal{Z}_t^q = \begin{bmatrix}
    z_t^T,z_{t-1}^T,\cdots,z_{t-q}^T,u_{t-1}^T,\cdots,u_{t-q}^T    \end{bmatrix}^T.
$
~Under Assumption \ref{a:obsv}, by taking the special case of a discrete system that is affine in control dynamics and transforming it into the information state form, the original partially observed optimal control problem can be equivalently posed as the following ``fully observed" optimal control problem in terms of the information state:
\begin{align}
    \bar{u}_t = \arg \min_{u_t}  \sum^{T-1}_{t=0} c(z_t, u_t) + c_T(z_T), \label{eq:OptProbZ} \\
    s.t. ~~~ \mathcal{Z}^q_{t+1} = \mathcal{Z}^q_t + \mathcal{F}(\mathcal{Z}^q_t)  \Delta t + \mathcal{G}(\mathcal{Z}^q_t) u_t \Delta t, \nonumber
\end{align}
where the constraint is the system dynamics in the information state form. The full development can be found in \cite{POD2C}. Then, we can extend our recent result on the globally optimal solution for the fully observed case \cite{tropaper2} to the above problem as shown in our companion work \cite{POD2C}.

\subsection{Open-Loop Optimal Trajectory Design using POD-ILQR}
\label{s:PODiLQR}
POD-iLQR takes advantage of iLQR in that the equations involved are given explicitly in terms of the LTV dynamics, which can be calculated in the information state form using a data-based LTV-ARMA identification method shown below.
Let us denote the nominal state and control trajectory by $\{\bar{x}_t, \bar{u}_t\}$ and the deviations from the nominal trajectory as $\delta x_t$ and $\delta u_t$, the LTV system linearized around the nominal trajectory can be modeled as:
$
    \delta x_t = A_{t-1} \delta x_{t-1} + B_{t-1} \delta u_{t-1}, ~ \delta z_t = C_t \delta x_t. \label{eq:LTV}
$
Let us denote the nominal information state and deviations as $\bar{\mathcal{Z}}_t$ and $\delta \mathcal{Z}_t = (\delta z_t,\delta z_{t-1},\cdots,\delta z_{t-q+1},\delta u_{t-1},\cdots,\delta u_{t-q+1})$, where $\delta z_t = z_t - \bar{z}_t$ is the deviations from the nominal observation at time $t$. Then we can show the following result:
\begin{proposition}\label{p:ARMA}
An ARMA model of the order $q$ given by: $\delta z_{t} = \alpha_{t-1}  \delta z_{t-1}+\cdots+\alpha_{t-q}   \delta z_{t-q}   + \beta_{t-1} \delta u_{t-1}+\cdots+ \beta_{t-q}   \delta u_{t-q},$ exactly fits the LTV system given in Eq.~\eqref{eq:LTV} if matrix $O^q = \begin{bmatrix} A_{t-q}^T...A_{t-2}^T C_{t-1}^T, & \cdots, & A_{t-q}^TC_{t-q+1}^T, & C_{t-q}^T \end{bmatrix}^T$ is full column rank. The exact ARMA parameters that match the LTV system can then be written as:
$
    \nonumber [\alpha_{t-1} ~|~  \alpha_{t-2} ~| \cdots|~\alpha_{t-q}] =  C_{t}A_{t-1}...A_{t-q}O^{q^+}, \\
    \nonumber [\beta_{t-1} ~|~  \beta_{t-2} ~| \cdots|~\beta_{t-q}] = - C_{t}A_{t-1}...A_{t-q}O^{q^+} G^q \\
   \nonumber  + \begin{bmatrix}
    C_{t}B_{t-1}  C_{t}A_{t-1}B_{t-2} ~ \cdots ~ C_{t}A_{t-1}...B_{t-q}
    \end{bmatrix}.
$
\end{proposition}
Thus with a $q$ that satisfies Assumption \ref{a:obsv}, there always exists an exact fit for the ARMA model. This allows us to write linearized models at each step along the nominal trajectory in terms of the ARMA parameters. The proof is in \cite{POD2C}.

With the above result, iLQR can be generalized into the proposed POD-iLQR method in the following:\\
\textbf{Forward Pass:}
Given the initial information state $\mathcal{Z}_0$ and the nominal control sequence $\{u_t\}_{t=0}^{T-1}$ of the current iteration, the system can be simulated for one rollout to get the nominal information state trajectory $(\bar{\mathcal{Z}}_t,\bar{u}_t)$ of the current iteration.\\
\textbf{LTV System Identification:}
Next, we can find the local LTV information state system around the nominal trajectory, which can be written as:
$
    \delta \mathcal{Z}_{t+1} = \mathcal{A}_t \delta \mathcal{Z}_t + \mathcal{B}_t \delta u_t
$
, where $\mathcal{A}_t$ and $\mathcal{B}_t$ are the linearization of the information state dynamics as shown in Eq.~\eqref{c4tab}.
To estimate $\mathcal{A}_t$ and $\mathcal{B}_t$, input-output data can be collected from simulated rollouts with control perturbation about the nominal trajectory. The LTV system can be written in ARMA parameters as
$ \delta z_{t}^{(j)} = \alpha_{t-1} \delta z_{t-1}^{(j)} +\cdots+\alpha_{t-q} \delta z_{t-q}^{(j)} + \beta_{t-1} \delta u_{t-1}^{(j)}+\cdots+ \beta_{t-q} \delta u_{t-q}^{(j)},
$
where $\delta u_{t}^{(j)} \sim \mathcal{N}(0,\sigma I)$ is the control perturbation at step $t$ for the $j$\textsuperscript{th} rollout. The ARMA parameters can be solved using linear least squares. Next, $\mathcal{A}_t$ and $\mathcal{B}_t$ can be obtained using the ARMA parameters. Note that the original partially observed system is transformed into a fully observed information state system.
Please check \cite{Wang_ICRA_2021} for more details.\\
\textbf{Backward Pass:}
Given the LTV system identified above, POD-iLQR computes a local optimal control by solving the discrete-time Riccati equation:
$
    \delta u_t=R^{-1}\mathcal{B}_t^T(-v_{t+1}-V_{t+1}(\mathcal{A}_t\delta \mathcal{Z}_t + \mathcal{B}_t\delta u_t))-\bar{u}_t,
$
which can be written in the linear feedback form $\delta u_t =-k_t-K_t\delta \mathcal{Z}_t$, where $k_t=(R+\mathcal{B}_t^T V_{t+1}\mathcal{B}_t)^{-1}(R\bar{u}_t+\mathcal{B}_t^T v_{t+1})$ and $K_t=(R+\mathcal{B}_t^T V_{t+1}\mathcal{B}_t)^{-1}\mathcal{B}^T_t V_{t+1}\mathcal{A}_t$, and
\begin{align}
    v_t &= l_{t,Z}+\mathcal{A}_t^T v_{t+1}-\mathcal{A}_t^T V_{t+1}\mathcal{B}_t(R+\mathcal{B}_t^T V_{t+1}\mathcal{B}_t)^{-1}\nonumber \\
    &\cdot(\mathcal{B}_t^T v_{t+1}+R\bar{u}_t) \label{eq:vt}
\end{align}
\begin{align}    V_t &= l_{t,ZZ}+\mathcal{A}_t^T (V_{t+1}^{-1}+\mathcal{B}_t R^{-1}\mathcal{B}_t^T)^{-1}\mathcal{A}_t\nonumber \\
    &=l_{t,ZZ}+\mathcal{A}_t^T V_{t+1}\mathcal{A}_t-\mathcal{A}_t^T V_{t+1}B_t(R+B_t^T V_{t+1}\mathcal{B}_t)^{-1} \label{eq:Vt} \nonumber \\
    &\cdot B_t^T V_{t+1}\mathcal{A}_t.
\end{align}
with the terminal conditions $v_T(x_T)=\frac{\partial c_T}{\partial \mathcal{Z}}|_{\mathcal{Z}_T}=\frac{\partial l}{\partial \mathcal{Z}}|_{\mathcal{Z}_T}$ and $V_T(x_T)=\nabla ^2_{ZZ}c_T|_{\mathcal{Z}_T}$. Given the terminal conditions and $(\mathcal{A}_t, \mathcal{B}_t)$, the sequence $v_t$ and $V_t$ can be computed in a backward sweep. Then, the corresponding gains $k_t$ and $K_t$ can be obtained for that trajectory. \\
\textbf{Trajectory Update:}
Given the gains from the backward pass, we can update the nominal control sequence as 
$
    \bar{u}_t^{k+1} = \bar{u}_t^k + \alpha k_t + K_t(\mathcal{Z}_t^{k+1}-\mathcal{Z}_t^k),
    \mathcal{Z}^{k+1}_0 = \mathcal{Z}^k_0,
$
where $\alpha$ is the line search parameter. By applying the control update at each step in the forward pass, we can obtain the updated nominal trajectory.

We iterate the above steps till $R\bar{u}_t+\mathcal{B}_t^T v_{t+1}\approx 0$ and obtain the open-loop optimal trajectory.



\begin{table*}[ht!]
\begin{align}\label{c4tab}
    \setcounter{MaxMatrixCols}{20}
    \underbrace{
    \begin{bmatrix} 
    \delta z_{t} \\ \delta z_{t-1} \\ \delta z_{t-2} \\ \vdots \\ \delta z_{t-q+1} \\ \hline \delta u_{t-1} \\ \delta u_{t-2} \\ \delta u_{t-3} \\ \vdots \\ \delta u_{t-q+1} 
    \end{bmatrix}}_{\delta \mathcal{Z}_t} = 
    \underbrace{\begin{bmatrix} 
    \alpha_{t-1} & \alpha_{t-2} & \cdots  &  \alpha_{t-q+1}  & \alpha_{t-q} & \vline &  \beta_{t-2} & \beta_{t-3} & \cdots  &  \beta_{t-q+1}  
    & \beta_{t-q}    \\  
    1 & 0 & \cdots & 0 & 0 & \vline & 0 & 0 & \cdots & 0 
    & 0    \\  
    0 & 1 & \cdots & 0 & 0 & \vline & 0 & 0 & \cdots & 0 
    & 0    \\  
    \vdots & & \ddots &  & \vdots & \vline & \vdots &  & \ddots & \vdots
    & 0    \\
    0 & 0 & \cdots & 1 & 0 & \vline & 0 & 0 & \cdots & 0 
    & 0    \\
    \hline
    0 & 0 & \cdots & 0 & 0 & \vline & 0 & 0 & \cdots & 0 
    & 0    \\
    0 & 0 & \cdots & 0 & 0 & \vline & 1 & 0 & \cdots & 0 
    & 0     \\  
     0 & 0 & \cdots & 0 & 0 & \vline & 0 & 1 & \cdots & 0 
    & 0     \\  
    \vdots & & \ddots &  & \vdots & \vline & \vdots &  & \ddots & & \vdots    \\
    0 & 0 & \cdots & 0 & 0 & \vline & 0 & 0 & \cdots & 1
    & 0
    \end{bmatrix} }_{\mathcal{A}_{t-1}}
    \underbrace{\begin{bmatrix} \delta z_{t-1} \\ \delta z_{t-2} \\ \vdots \\ \delta z_{t-q+1}  \\ \delta z_{t-q} \\ \hline  \delta u_{t-2} \\ \delta u_{t-3} \\   \vdots \\ \delta u_{t-q+1} \\ \delta u_{t-q} \end{bmatrix}}_{\delta \mathcal{Z}_{t-1}} 
    + \underbrace{\begin{bmatrix} \beta_{t-1} \\ 0 \\ \vdots \\ 0 \\ 0 \\ \hline 1 \\ 0  \\ \vdots \\  0 \\0 \end{bmatrix}}_{\mathcal{B}_{t-1}} \delta u_{t-1} 
\end{align}
\end{table*}

\section{Learning under Uncertainty with POD-iLQR}
In our previous work, the POD-iLQR was tested in simulation without any process noise or measurement noise. The main focus of this work is to extend POD-iLQR and solve the partially observed optimal control problem in the presence of noise and unknown system dynamics. 

\subsection{POD-ILQR Extension for Learning Under Uncertainty}
\label{s:modify}
In the following, we study the problem of learning under uncertainty with POD-iLQR in two distinct cases.\\
{\textbf{1) Fully Observed Case with Process Noise}:} We assume that there is process noise in the system dynamics but the full state measurements are perfect, i.e., 
\begin{equation}
\label{eq:fullsys}
    x_{t+1} = f(x_t,u_t)+\omega_t, \; z_t = x_t.
\end{equation}
With process noise, the simulation in the system identification can deviate from the nominal and adversely affect the accuracy of the identified model. The inaccurate model then leads to difficulty in convergence. To keep the trajectory close to the nominal trajectory, we make the following modification to POD-iLQR:\\
\textit{\textbf{LTV System Identification:}} A feedback term in control is added using the gain $K$ obtained from the previous backward pass, i.e., $u_t = \bar{u}_t+\delta u_t+K_t(\delta z_t)$, where $\delta z_t = z_t - \bar{z}_t$ is the deviated full state measurement and $\delta u_t$ is the input perturbation sampled from a zero-mean i.i.d. process. The idea is that the feedback should keep the rollouts close to the nominal to ensure accurate models.\\ 
\noindent \textbf{2) Partially Observed Case:} We assume that both the system dynamics and the sensors are corrupted by noise. In addition, only a subset of the states is measured, i.e., 
\begin{equation}
\label{eq:partialsys}
    x_{t+1} = f(x_t,u_t)+\omega_t, \; z_t = C_tx_t+v_t,
\end{equation}
where $C_t \in \mathcal{R}^{n_z \times n_x}, n_z < n_x$ and $n_z$ is the number of outputs.
In this case, we also need to keep the simulated trajectories close to the nominal and thus we make the following modifications to POD-iLQR:\\
\textit{\textbf{LTV System Identification:}} A feedback term is added in the ARMA fitting step using the feedback gain from the last backward pass, i.e., $u_t = \bar{u}_t+\delta u_t+K_t(\delta \mathcal{Z}_t)$, where $\delta \mathcal{Z}_t = \mathcal{Z}_t - \bar{\mathcal{Z}}_t$ is the deviated information state vector as shown in Eq.~\eqref{c4tab} and $\delta u_t$ is the input perturbation sampled from a zero-mean i.i.d. process. \\
\noindent \textbf{3) Brute Force Averaging:} In the following section, we will show that even when the feedback term is added in the partially observed case, the modified POD-iLQR algorithm still gives a biased result instead of the true optimum due to the partial observation, the process, and the measurement noise. To remove this bias, we utilize the assumption that the process and measurement noise is zero-mean and make the following modifications to the POD-iLQR algorithm:\\
\textit{\textbf{Forward Pass:}} In each iteration, the forward pass simulation is run for $n_s$ number of rollouts. Then we take the average trajectory of the rollouts as the updated nominal trajectory.
\textit{\textbf{LTV System Identification:}} For each sequence of control perturbation $\{\delta u_t\}_{t=0}^{T-1}$, we run the simulations for $n_s$ number of rollouts and take the average. Also, a feedback term is added in control to keep the trajectories close to the nominal trajectory. As the noise is assumed to be zero mean, the averaged trajectories are used in the least square to identify the ARMA model. 

\subsection{Convergence Analysis of POD-iLQR under Uncertainty}
According to the results in our previous work \cite{POD2C}, the POD-iLQR algorithm in the deterministic environment is guaranteed to find the unique global minimum of the open-loop problem in Eq.~\eqref{eq:OptProbZ}. In this section, we analyze the convergence and optimality of the modified POD-iLQR in the two cases described above.

\subsubsection{Global Convergence of POD-iLQR in the Fully Observed Case with Process Noise}
As mentioned in Section \ref{s:modify}, in the system simulations of the LTV system identification step, we implement a feedback term in control to ensure that the trajectories are close to the nominal. With this modified POD-iLQR method, we have the following result:
\begin{lemma}\label{lm:arma_full}
The ARMA model identified in the fully observed case using the method described in Section \ref{s:modify}.1 is unbiased with respect to the true LTV model.
\end{lemma}

\begin{proof}
Due to full state observation, the matrix $O^q = \begin{bmatrix} A_{t-q}^T...A_{t-2}^T C_{t-1}^T, & \cdots, & A_{t-q}^TC_{t-q+1}^T, & C_{t-q}^T \end{bmatrix}^T$ is full column rank. Thus an ARMA model with $q = 1$ can exactly fit the LTV system in Eq.~\eqref{eq:LTV} according to Proposition \ref{p:ARMA}. 
Then Eq.~\eqref{eq:LTV} and \eqref{eq:fullsys}
leads to:
$
    \delta z_{t+1} = (A_t+B_tK_t)\delta z_{t}+B_t \delta u_t+\omega_t.
$
Notice that $\omega_t$ is uncorrelated with $\delta z_t$ and $\delta u_t$. Thus using the LTV system identification described in Section \ref{s:PODiLQR}, the noise terms go to zero as the number of rollouts $n_s$ increases and the least square result becomes:
$
    \left[\hat{A}_t ~\vdots~ \hat{B}_t\right] = \left[A_t+B_tK_t ~\vdots~ B_t\right].
$
As $K_t$ is known from the last backward pass, it is trivial to recover $A_t, B_t$. Thus the ARMA model equals the true LTV model in Eq.~\eqref{eq:LTV}.
\end{proof}

\begin{theorem}\label{InformState_optimality_full}
Let the cost functions $l(\cdot)$, $c_T(\cdot)$, the drift $f(\cdot)$ and the input influence function $g(\cdot)$ be $\mathcal{C}^2$, i.e., twice continuously differentiable. Assume $f(\cdot)$ is affine in control. The POD-iLQR algorithm started at a feasible initial information state $\mathcal{Z}_0$ converges to the unique global minimum of the open-loop problem in Eq.~\eqref{eq:OptProbZ} when applied to the fully observed system with process noise in Eq.~\eqref{eq:fullsys}.
\end{theorem}
\begin{proof}
From Lemma \ref{lm:arma_full}, we know that the identified LTV model in the fully observed case is the same as in the noiseless case for small enough noise. Thus the results from the backward pass are also identical. Let us denote the update direction calculated in the forward pass under process noise as $d_{s,t}=[\delta \mathcal{Z}_{t+1}'~~\delta \mathcal{Z}_t'~~\delta u_t']'$, gradient of the system dynamics constraint function as:
$
    \nabla h(\bar{\mathcal{Z}}_{t+1},\bar{\mathcal{Z}}_t,\bar{u}_t)=[I~~- \mathcal{A}_{t}~~-\mathcal{B}_{t}].
$
Let $\mathcal{F}^t$ denote the history of the algorithm till time $t$. Then, due to the zero mean noise $\omega_t$, it is easy to see that the expected descent direction conditioned on the history $\mathcal{F}^t$:
$
E[d_{s,t}/ \mathcal{F}^t] = \bar{d}_{s,t},
$
where $\bar{d}_{s,t}$ denotes the true update direction (without noise). We know from Lemma 2 in \cite{tropaper2}, that $\bar{d}_{s,t}$ is a descent direction of the cost function, i.e., $\bar{d}_{s,t}'\nabla J'_t$ is always negative. Thus the expected update direction is a descent direction. 
Then using the line search condition of iLQR, similar to Theorem 1 in \cite{tropaper2}, it can be shown that:
$
    E[J_{t+1}/ \mathcal{F}^t] \leq J_t - \bar{\beta}_t ||\nabla J_t||||\bar{d}_{s,t}||,
$
for some $\bar{\beta}_t > 0$. Then, using the Supermartingale Convergence Theorem \cite{neuroDPbook}, it follows that, almost surely:
$
    \sum_t \bar{\beta}_t ||\nabla J_t|| < \infty,
$
which implies that $\nabla J_t \rightarrow 0$ almost surely, i.e., the algorithm converges to a stationary point of the cost function. Next, using Theorem 2 of \cite{tropaper2}
and Theorem III.1 of \cite{POD2C}, POD-iLQR is guaranteed to converge to the global minimum of the open-loop problem in Eq.~\eqref{eq:OptProbZ}.
\end{proof}

\subsubsection{Biased Nature of POD-iLQR in the Partially Observed Case}
For the partially observed case, process noise $\omega_t$ is added to the system dynamics simulation in both the forward and backward pass. The measurement noise $v_t$ is added to the measured states as shown in Eq.~\eqref{eq:partialsys}. According to Section \ref{s:infostate}, we need to choose a large enough $q$ such that matrix $O^q$ is full column rank. Also, a feedback term is added in the ARMA fitting step of the backward pass to keep the trajectory close to the nominal trajectory. The feedback gain $K_t$ is obtained from the last backward pass. With these modifications, we have the following negative result:
\begin{lemma}\label{lm:partialdir}
If directly applied to the partially observed system in Eq.~\eqref{eq:partialsys}, the backward pass shown in Section \ref{s:PODiLQR} generates a biased update direction.
\end{lemma}
\begin{proof}
Starting from Eq.~\eqref{eq:partialsys}, we can write the output equation for past $q$ timesteps.
Assuming that the $q$ we choose satisfies Assumption \ref{a:obsv}, we can solve for the unique solution of $\delta x_{t-q}$ as:
$
\delta x_{t-q} = O^{q^+} (\delta Z_t^q - G^q \delta U_t^q-G^q_{\omega} \Omega^q_t-V_t^q).
$
Now, the system output at time $t$ can be written as:
\begin{align}
    \delta z_{t} &= C_{t}A_{t-1}...A_{t-q}\delta x_{t-q} \nonumber \\
    &+\begin{bmatrix}
    C_{t}B_{t-1} & C_{t}A_{t-1}B_{t-2} ~ \cdots ~ C_{t}A_{t-1}...B_{t-q}
    \end{bmatrix}\delta U_t^q,  \nonumber \\
    &+\begin{bmatrix}
    C_{t} & C_{t}A_{t-1}~ \cdots ~ C_{t}A_{t-1}...A_{t-q+1}
    \end{bmatrix}\Omega^q_t+v_t.
\end{align}
Let us denote
\begin{align}
    \alpha_t &= C_{t}A_{t-1}...A_{t-q}O^{q^+}, \nonumber \\
    \beta_t&=\begin{bmatrix}
    C_{t}B_{t-1} & C_{t}A_{t-1}B_{t-2} ~ \cdots ~ C_{t}A_{t-1}...B_{t-q}
    \end{bmatrix}-\alpha_t G^q, \nonumber \\
    \beta_t^D&=\begin{bmatrix}
    C_{t} & C_{t}A_{t-1} ~ \cdots ~ C_{t}A_{t-1}...A_{t-q+1}\end{bmatrix}-\alpha_tG^q_{\omega},\nonumber
\end{align}
and the unique solution for $\delta x_{t-q}$ is substituted to get:
$
    \delta z_{t} = \alpha_t \delta Z_t^q+\beta_t \delta U_t^q +\beta_t^D\Omega^q_t-\alpha_tV_t^q+v_t.
$
By taking $n_s$ number of rollouts and applying the linear least square as described in Section \ref{s:modify}.2, the estimated ARMA model parameters can be written as:
\begin{align}\label{eq:partialsysid}
    \begin{bmatrix}\hat{\alpha}_t &\hat{\beta}_t\end{bmatrix} &= \begin{bmatrix}\alpha_t& \beta_t\end{bmatrix}+(\beta_t^D\begin{bmatrix}\Omega^{q,(1)}_t &\Omega^{q,(2)}_t&\cdots&\Omega^{q,(n_s)}_t\end{bmatrix} \nonumber\\
    &-\alpha_t\begin{bmatrix}V_t^{q,(1)} &V^{q,(2)}_t&\cdots&V^{q,(n_s)}_t\end{bmatrix}\nonumber\\
    &+\begin{bmatrix}v_t^{(1)} &v^{(2)}_t&\cdots&v^{(n_s)}_t\end{bmatrix})\mathbb{X} \T(\mathbb{X}\mathbb{X} \T)^{-1},
\end{align}
where in this case,
\begin{align}
    \mathbb{X} &= \begin{bmatrix} \delta {Z_t^{q,(1)}}& \delta {Z_t^{q,(2)}}& \ldots &\delta {Z_t^{q,(n_s)}} \\ \delta {U_t^{q,(1)}}& \delta {U_t^{q,(2)}}& \ldots &\delta {U_t^{q,(n_s)}} \end{bmatrix}.
\end{align}
As $\delta {Z_t^{q}}$ is correlated with $\Omega^{q}_t$ and $V_t^q$, the second term on the RHS of Eq.~\eqref{eq:partialsysid} is nonzero. Thus the estimated ARMA parameters are biased from the true values in $\begin{bmatrix}\alpha_t& \beta_t\end{bmatrix}$. Further, the update direction from Section \ref{s:modify} is biased.
\end{proof}

With the biased update direction, the POD-iLQR algorithm can no longer converge to the true minimum although the feedback term is added to the backward pass. In this case, multiple rollouts have to be averaged to recover the convergence to the true minimum. In the next section, we show empirical evidence for Theorem \ref{InformState_optimality_full} and Lemma \ref{lm:partialdir}.

\section{Empirical Results}
We use MuJoCo, a physics engine \cite{todorov2012mujoco}, as a blackbox to collect the data needed for the open-loop nominal trajectory design and the closed-loop feedback gain. We verified our results on two nonlinear systems with their initial configuration shown in Fig.~\ref{figinit}. All simulations are conducted on a machine with the following specifications: AMD Ryzen 3700X 8-Core CPU@3.59 GHz, with 16 GB RAM, with no multi-threading. 

\begin{figure}[!ht]
\begin{multicols}{2}
    \hspace{1cm}   
    \subfloat{\includegraphics[width=.76\linewidth]{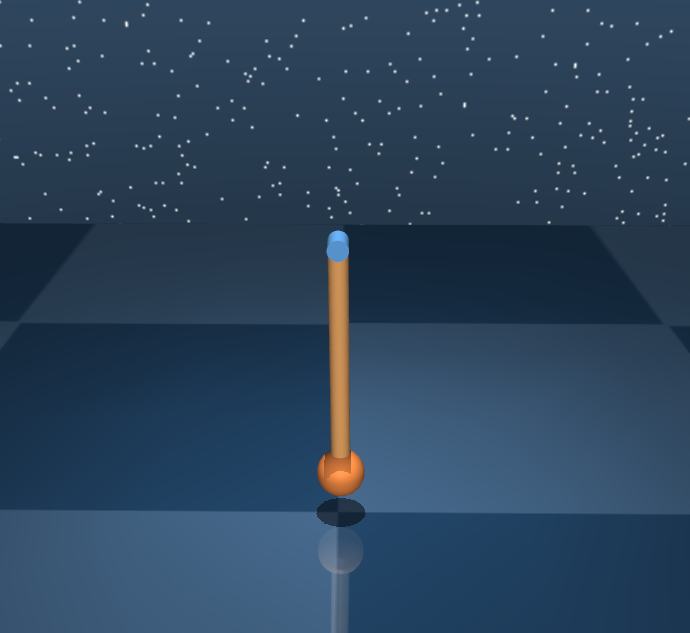}}
    \subfloat{\includegraphics[width=.80\linewidth]{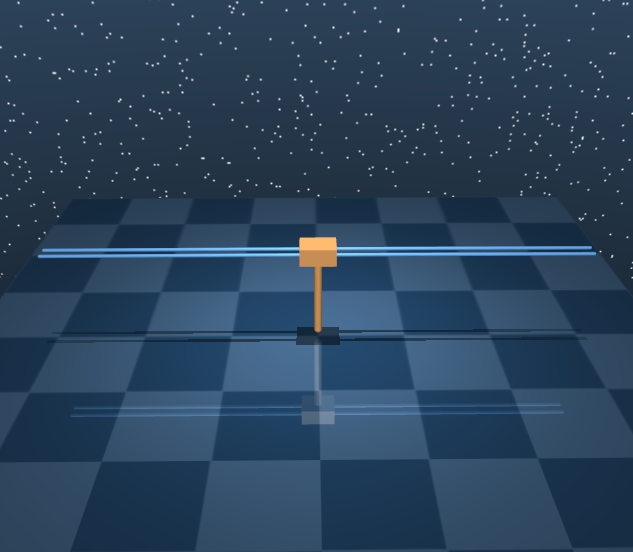}}  
\end{multicols}
\caption{Models simulated in MuJoCo in their initial states.}
\label{figinit}
\end{figure}

\vspace{-0.5cm}
\subsection{Model Description}
Here we provide details of the MuJoCo models used in our simulations \cite{tunyasuvunakool2020}.

\noindent \textbf{Pendulum:} The single pendulum model is a pole hinged to a fixed point. There are two state variables: angle and angular rate of the pole. The task is to swing up and balance the pole in the upright position.\\
\noindent \textbf{Cart-Pole:} The four-dimensional under-actuated cart-pole model includes a cart moving on the x-axis and a pole linked to it with a hinge joint. The only actuation is the force on the cart. The state comprises the angle of the pole, the cart's horizontal position, and their rates. The task is to swing up and balance the pole in the middle of the rail within a given horizon.\\

\begin{figure}[!ht]
\centering
    \includegraphics[width=1\linewidth]{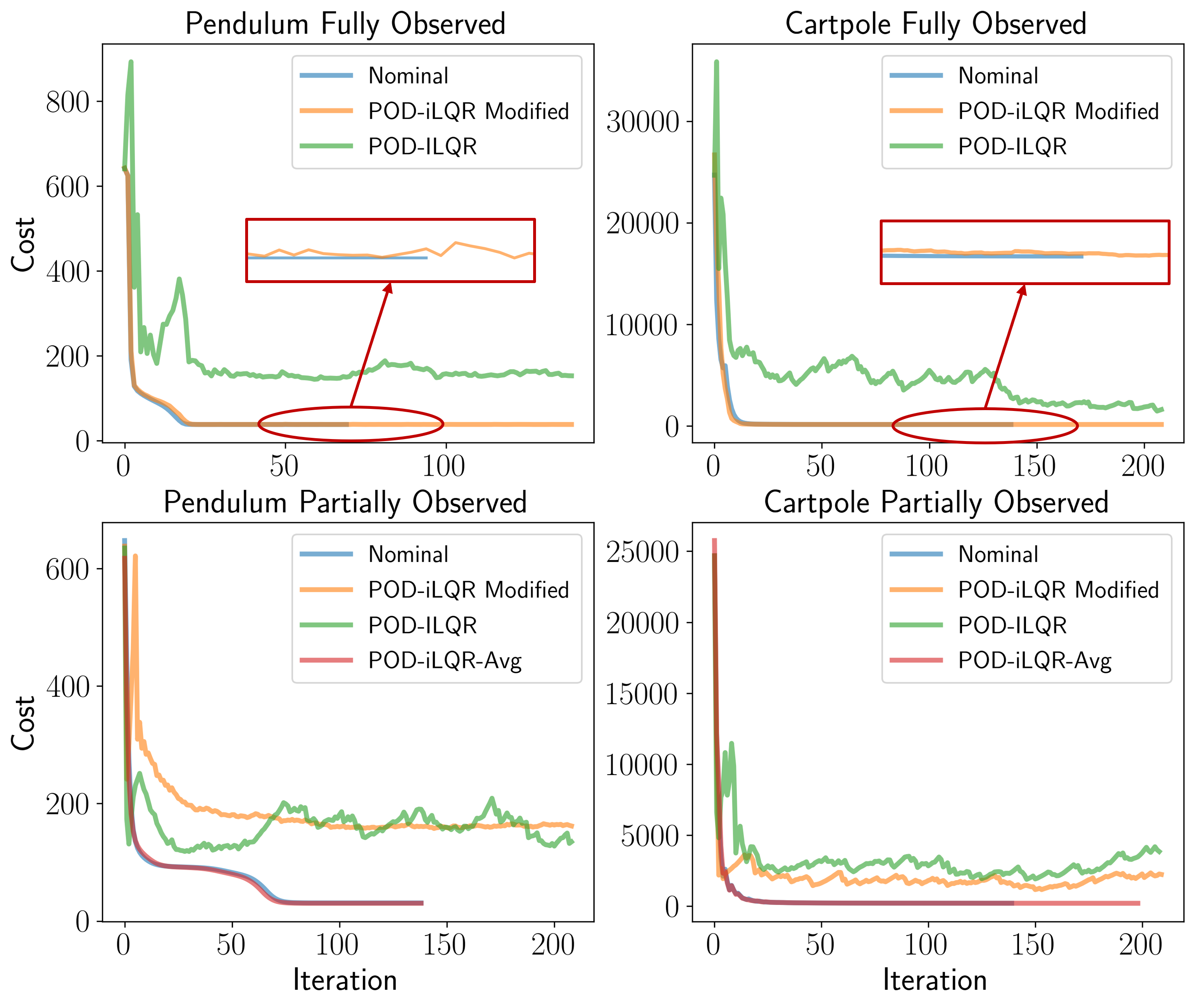}
\caption{Convergence comparison in fully and partially observed cases.}\vspace{-0.5cm}
\label{figconvfullpartial}
\end{figure}

\subsection{POD-iLQR in the Fully Observed Case}
As we assume perfect measurements in the fully observed case, the process noise is the only uncertainty we add to the dynamics. In addition, we sample the initial state deviation $\delta x_0$ from a zero-mean random process. The standard deviation of the process noise $\omega_t$ is 10\% w.r.t. the standard deviation of the initial state deviation $\delta x_0$. For each system tested, we run POD-iLQR as is in the noiseless system as well as in the fully observed system with process noise. Then we run the POD-iLQR with the modifications proposed in Section \ref{s:modify}. The number of steps in the horizon is fixed, so in the data collection step, each rollout takes the same number of steps. In the cost function, the running cost $l(x_t)=x_t'Qx_t$, where $x_t$ is the error between the current state and the target state at time $t$. The cost parameter $Q$ remains the same throughout the horizon except for the terminal step. The termination criterion is that the convergence rate is lower than a threshold or it reaches the max iteration number. We compare the cost convergence curves in the first row of Fig.~\ref{figconvfullpartial}. The ``nominal" curve shows the cost convergence of applying POD-iLQR on the full state noiseless case. The curve labeled ``POD-iLQR" shows the cost convergence in the full state case under process noise using the original POD-iLQR algorithm. The curve labeled ``POD-iLQR Modified" shows the cost convergence in the full state case under process noise using the modified POD-iLQR. From the plot, the original POD-iLQR could not converge to the true optimum and the other two curves almost overlap each other and converge to the same result. The ``POD-iLQR Modified" curve in the zoomed-in view has some ups and downs due to the process noise in the forward pass. This verifies the proof in Theorem \ref{InformState_optimality_full} which shows that the expected update direction, not the actual update direction is a descending direction. Thus in the fully observed case, the modified POD-iLQR is guaranteed to converge to the global minimum. Notice that the total number of rollouts needed under noise is larger than in the noiseless case to make sure that the correlation goes to zero as shown in Lemma \ref{lm:arma_full}. 

\subsection{POD-iLQR in the Partially Observed Case}
In the partially observed case, we only measure the positions, not their rates. For the pendulum, the observed state is the angle of the pole. In the cartpole, we only measure the position of the cart and the angle of the pole. To simulate the measurements in simulation, we add sensor noise $v_t$ to the measurements. Thus there are both process and measurement noise in the simulation. The standard deviations of the process noise $\omega_t$ and the measurement noise $v_t$ are both 10\% w.r.t. the standard deviation of the initial state deviation $\delta x_0$. We found that in this observation setting, $q=2$ satisfies the observability assumption. So in the LTV-ARMA system identification step, we fit ARMA models with $q=2$ for both the pendulum and cartpole. To evaluate the cost function, we use the measurements instead of the states. Similar to the full-state observation case, the cost function is quadratic in the information state and the control input. Three cost curves are shown in the second row of Fig.~\ref{figconvfullpartial}. The curve labeled ``nominal" shows the cost convergence of POD-iLQR in the partially observed but noiseless environment. For the curve labeled ``POD-iLQR", we directly apply the unmodified POD-iLQR algorithm and in the ``POD-iLQR Modified" case, we use the modified POD-iLQR. In the case labeled ``POD-iLQR-Avg", we run multiple rollouts for each set of control perturbation $\delta u_t$ and used the averaged trajectory in the ARMA model fitting step to average out the noise. From the plots, it is shown that in both the pendulum and the cartpole, the cost curves of the averaging method match the nominal cost curve and they converge to the same result. The outlier curve is from the experiment where we applied POD-iLQR or modified POD-iLQR without averaging. Due to the noise corrupted system dynamics and measurements, the cost curve has more oscillation and failed to converge to the true minimum. Thus if directly implemented on real robots without averaging, POD-iLQR will generate a biased result even with the modification proposed in Section \ref{s:modify}. To make sure the noise is averaged out in the averaging method, the total number of rollouts needed in the ARMA model fitting is increased from $n_s$ to $n_s \times n_{avg}$, where $n_s$ is the number of rollouts in one ARMA model fitting without averaging and $n_{avg}$ is the number of rollouts needed for each control perturbation set. And the total time taken during training will be much higher than the case without averaging.

\section{Conclusions}
This paper considers an optimal motion planning trajectory design algorithm for partially observed systems called the POD-iLQR introduced in our prior work. 
The main focus of this paper is to analyze its performance under uncertainty and pave the way for future use on real-world robots.
The algorithm is proved to converge to the global minimum in the fully observed case with only process noise.
It is shown that the algorithm is biased and does not converge to the global minimum for partially observed systems with both process and measurement noise. In this case, multiple rollouts need to be averaged to recover optimality and for convergence in the ARMA model identification step at the expense of a longer training time.
The empirical results are shown to verify the analysis.
In our opinion, this algorithm has advantages in optimality and training efficiency when applied to real-world robots. The actual performance on real systems will be explored in future work.

\printbibliography
\end{document}